\documentclass[12pt]{colt2017} 
\usepackage{hyperref}
\usepackage{url}
\usepackage{times}
\usepackage{xcolor}

\usepackage{wrapfig,enumitem}

\usepackage{amsmath,amssymb,algorithm,bbm}

\newtheorem{assumption}{Assumption}

\newcommand{\loss}{\ell}
\newcommand{\Loss}{\mathcal{L}}

\newcommand{\X}{\mathcal{X}}

\newcommand{\tloss}{\wt{\ell}}

\newcommand{\real}{\mathbb{R}}

\newcommand{\Qstar}{\mathcal{Q}^*}

\newcommand{\PP}[1]{\mathbb{P}\left[#1\right]}
\newcommand{\EE}[1]{\mathbb{E}\left[#1\right]}

\newcommand{\ev}[1]{\left\{#1\right\}}
\newcommand{\pa}[1]{\left(#1\right)}
\newcommand{\bpa}[1]{\bigl(#1\bigr)}
\newcommand{\Bpa}[1]{\Bigl(#1\Bigr)}

\renewcommand{\th}{\ensuremath{^{\mathrm{th}}}}
\def\argmin{\mathop{\mbox{ \rm arg\,min}}}

\newcommand{\bloss}{\overline{\loss}}

\newcommand{\bc}{\overline{c}}

\newcommand{\wt}{\widetilde}

\newcommand{\ra}{\rightarrow}
\newcommand{\transpose}{^\mathsf{\scriptscriptstyle T}}

\newcommand{\norm}[1]{\left\|#1\right\|}
\newcommand{\onenorm}[1]{\norm{#1}_1}

\newcommand{\infnorm}[1]{\norm{#1}_\infty}
\newcommand{\spannorm}[1]{\norm{#1}_s}

\newcommand{\kl}[2]{D\pa{\left.#1\right\|#2}}

\usepackage{todonotes}
\definecolor{PalePurp}{rgb}{0.66,0.57,0.66}

\newenvironment{proofsketch}%
{%
 \par\noindent{\bfseries\upshape Proof sketch\ }%
}%
{\jmlrQED}

\title[Fast rates for online learning in LMDPs]{Fast rates for online learning in 
\\Linearly Solvable Markov Decision Processes}

\author[Neu and G\'omez]{\Name{Gergely Neu} \Email{gergely.neu@gmail.com}\\
 \addr Universitat Pompeu Fabra, Barcelona, Spain
 \AND
 \Name{Vicen\c c  G\'omez} \Email{vicen.gomez@upf.edu}\\
 \addr Universitat Pompeu Fabra, Barcelona, Spain
}

\begin{document}

\maketitle

\begin{abstract}
We study the problem of online learning in a class of Markov decision processes known as \emph{linearly solvable MDPs}. In the stationary 
version of this problem, a learner interacts with its environment by directly controlling the state transitions, attempting to balance 
a fixed state-dependent cost and a certain smooth cost penalizing extreme control inputs. In the current paper, we consider an online 
setting where the state costs may change arbitrarily between consecutive rounds, and the learner only observes the costs at the end of 
each respective round. We are interested in constructing algorithms for the learner that guarantee small regret against the best 
stationary control policy chosen in full knowledge of the cost sequence. Our main result is showing that the smoothness of the control cost 
enables the simple algorithm of \emph{following the leader} to achieve a regret of order $\log^2 T$ after $T$ rounds, vastly
improving on the best known regret bound of order $T^{3/4}$ for this setting.
\end{abstract}


\section{Introduction}
We consider the problem of online learning in Markov decision processes (MDPs) where a learner sequentially interacts with an environment 
by repeatedly taking actions that influence the future states of the environment while incurring some immediate costs. The goal of the 
learner is to choose its actions in a way that the accumulated costs are as small as possible. Several variants of this problem have been 
well-studied in the literature, primarily in the case where the costs are assumed to be independent and identically distributed 
\citep{sutton,Puterman1994,ndp,Sze10}. In the current paper, we consider the case where the costs are generated by an arbitrary external 
process and the learner aims to minimize its total loss during the learning procedure---conforming to the learning paradigm known as 
\emph{online learning} \citep{CBLu06:Book,SS12}. In the online-learning framework, the performance of the learner is measured in terms of 
the \emph{regret} defined as the gap between the total costs incurred by the learner and the total costs of the best comparator chosen
from a pre-specified class of strategies. In the case of online learning in MDPs, a natural class of strategies is the set of all 
state-feedback policies: several works studied minimizing regret against this class both in the 
stationary-cost~\citep{bartlett09regal,jaksch10ucrl,AYSze11} and the non-stochastic 
setting~\citep{even-dar09OnlineMDP,yu09ArbitraryRewards,neu10o-ssp,neu12ssp-trans,ZiNe13,DGS14,neu14o-mdp-full,AYBK14}. In the 
non-stochastic setting, most works 
consider MDPs with unstructured, finite state spaces and guarantee that the regret increases no faster than~$O(\sqrt{T})$ as the number of 
interaction rounds $T$ grows large. A notable exception is the work of \citet{AYBK14}, who consider the special 
case of (continuous-state) linear-quadratic control with arbitrarily changing target states, and propose an algorithm that guarantees a 
regret bound of $O(\log^2 T)$.

In the present paper, we study another special class of MDPs that turns out to allow fast rates. Specifically, we consider the class of 
so-called \emph{linearly solvable MDPs} (in short, LMDPs), first proposed and named so by \citet{Tod06}.
This class takes its name 
after the special property that the Bellman optimality equations characterizing
the optimal behavior policy take the form of a system of linear equations,
which makes optimization remarkably straightforward in such problems. The
continuous formulation (in both space and time) was discovered independently
by~\citet{Kap05} and is known as \emph{path integral control}.
LMDPs have many interesting properties. For example, optimal control laws for
LMDPs can be linearly combined to derive composite optimal control laws
efficiently~\citep{Tod09}. Also, the inverse optimal control problem in LMDPs can be expressed as a convex optimization problem~\citep{dvijotham2010inverse}.
LMDPs generalize an existing duality between optimal control computation and Bayesian inference~\citep{todorov2008general}.
Indeed, the popular belief propagation algorithm used in dynamic probabilistic graphical models is
equivalent the the power iteration method used to solve LMDPs~\citep{KGO12}.

The LMDP framework has found applications in robotics~\citep{latentkl,ariki}, crowdsourcing~\citep{Abbasi}, and controlling the growth 
dynamics of complex networks~\citep{klnetwork}. The related path integral control framework of \citet{Kap05} has been applied in several 
real-world tasks, including robot navigation~\citep{KinjoFN2013}, motor skill 
reinforcement learning~\citep{TheodorouJMLR10a,rombokas2013reinforcement,pireps}, aggressive car maneuvering~\citep{aggressive} or 
autonomous flight of teams of quadrotors~\citep{pi_uavs}.

In the present paper, we show that besides the aforementioned properties, the structure of LMDPs also enables constructing efficient online 
learning procedures with very low regret. In particular, we show that, under some mild assumptions on the structure of the LMDP, the 
(conceptually) simplest online learning strategy of \emph{following the leader} guarantees a regret of order $\log^2 T$, vastly improving 
over the best known previous result by \citet*{GRW12}, who prove a regret bound of order $T^{3/4+\epsilon}$ for arbitrarily small 
$\epsilon>0$ under the same assumptions. Our approach is based on the observation that the optimal control law arising from the LMDP 
structure is a smooth function of the underlying cost function, enabling rapid learning without any regularization whatsoever.

The rest of the paper is organized as follows. Section~\ref{sec:bck} introduces the formalism of LMDPs and summarizes some basic facts that 
our technical content is going to rely on. Section~\ref{sec:online} describes our online learning model. Our learning algorithm is 
described in Section~\ref{sec:ftl} and analyzed in Section~\ref{sec:analysis}. Finally, we draw conclusions in Section~\ref{sec:discussion}.

\paragraph{Notation.} We will consider several real-valued functions over a finite state-space $\X$, and we will often treat these functions 
as finite-dimensional (column) vectors endowed with the usual definitions of the $\ell_p$ norms. The set of probability distributions over 
$\X$ will be denoted as $\Delta(\X)$. Indefinite sums with running variables $x,y$ or $s$ are understood to run through all $\X$.

\section{Background on linearly solvable MDPs}\label{sec:bck}
This section serves as a quick introduction into the formalism of linearly solvable MDPs (LMDPs, \cite{Tod06}). These decision processes 
are defined by the tuple $\ev{\X,P,c}$, where $\X$ is a finite set of \emph{states}, $P:\X \ra \Delta(\X)$ is a transition kernel called 
the \emph{passive dynamics} (with $P(x'|x)$ being the probability of the process moving to state $x'$ given the previous state $x$) and 
$c:\X\ra[0,1]$ is the \emph{state-cost function}. Our Markov decision process is a sequential decision-making problem where the initial 
state $X_0$ is drawn from some distribution $\mu_0$, and the following steps are repeated for an indefinite number of rounds $t=1,2,\dots$:
\begin{enumerate}[leftmargin=.7cm]
 \item The learner chooses a transition kernel $Q_t:\X\ra \Delta(\X)$ satisfying $\mathop{supp}Q_t(\cdot|x) \subseteq 
\mathop{supp}P(\cdot|x)$ 
for all $x\in\X$.
 \item The learner observes $X_t\in\X$ and draws the next state $X_{t+1}\sim Q(\cdot|X_t)$.
 \item The learner incurs the cost
 \[
  \ell(X_t,Q_t) = c(X_t) + \kl{Q_t(\cdot|X_t)}{P(\cdot|X_t)},
 \]
 where $\kl{q}{p}$ is the relative entropy (or Kullback-Leibler divergence) between the probability distributions $p$ 
and $q$ defined as $\kl{q}{p}=\sum_x q(x)\log \frac{q(x)}{p(x)}$.
\end{enumerate}

The state-cost function $c$ should be thought of as specifying the objective for the learner in the MDP, while the relative-entropy term 
governs the costs associated with significant deviations from the passive dynamics. Accordingly, we refer to this component as the 
\emph{control 
cost}. A central question in the theory of Markov decision problems is finding a behavior policy that minimizes (some notion of) the 
long-term total costs. In this paper, we consider the problem of \emph{minimizing the long-term average cost-per-stage} 
$\lim\sup_{T\ra\infty} \frac 1T \sum_{t=1}^T \loss(X_t,Q_t)$. Assuming that the passive dynamics $P$ is aperiodic and irreducible, this 
limit is minimized by a \emph{stationary} policy $Q$ (see, e.g., \citet[Sec.~8.4.4]{Puterman1994}). Below, we provide two distinct 
derivations for the optimal stationary policy that minimizes the average costs under this assumption.

\subsection{The Bellman equations}\label{sec:bellman}
We first take an approach rooted in dynamic programming \citep{Ber07:DPbookVol2}, following \citet{Tod06}. Under our assumptions, the 
optimal stationary policy minimizing the average cost is given by finding the solution to the Bellman optimality equation
\begin{equation}\label{eq:bellman_V}
 v(x) = c(x) - \lambda + \min_{q\in\Delta(\X)} \ev{\kl{q}{P(\cdot|x)} + \sum_{x'} q(x') v(x')}
\end{equation}
for all $x\in\X$, where $v$ is called the \emph{optimal value function} and $\lambda\in\real$ is the average cost associated with the 
optimal policy\footnote{This solution is guaranteed to be unique up to a constant shift of the values: if $v$ is a solution, then so is 
$v + a$ for any $a\in \real$. Unless stated otherwise, we will assume that $v$ is such that $v(x_0) = 0$ holds for a fixed state 
$x_0\in\X$.}.  Linearly solvable MDPs get their name from the fact that the Bellman optimality 
equation can be rewritten in a simple 
linear form. To see this, observe that by elementary calculations involving Lagrange multipliers, we have
\begin{align*}
 \min_{q\in\Delta(\X)} \ev{\kl{q}{P(\cdot|x)} + \sum_{x'} q(x') v(x')} =& -\log \sum_{x'} P(x'|x) e^{-v(x')},
\end{align*}
so, after defining the exponentiated value function $z(x) = e^{-v(x)}$ for all $x$, plugging into Equation~\eqref{eq:bellman_V} and 
exponentiating both sides gives
\begin{equation}\label{eq:bellman_Z}
 z(x) = e^{\lambda - c(x)} \sum_{x'} P(x'|x) z(x').
\end{equation}
Rewriting the above set of equations in matrix form, we obtain the linear equations
\[
 e^{-\lambda} z = GPz,
\]
where $G$ is a diagonal matrix with $G_{ii} = e^{-c(i)}$. By the \emph{Perron-Frobenius theorem} (see, e.g., Chapter~8 of \cite{Mey00}) 
concerning positive matrices, the above system of linear equations has a unique\footnote{As in the case of the Bellman equations, this 
solution is unique up to a \emph{scaling} of $z$.} solution satisfying $z(x)\ge 0$ for all $x$, and this 
eigenvector corresponds to the largest eigenvalue $e^{-\lambda}$ of $GP$. Since the solution of the Bellman optimality equation 
\eqref{eq:bellman_V} is unique (up to a constant shift corresponding to a constant scaling of $z$), we obtain that $\lambda$ is the average 
cost of the optimal policy. In summary, the Bellman optimality 
equation takes the form of a \emph{Perron--Frobenius eigenvalue problem}, which can be efficiently solved by iterative methods such as the 
well-known power method for finding top eigenvectors. Finally, getting back to the basic form~\eqref{eq:bellman_V} of the Bellman equations, 
we can conclude after simple calculations that the optimal policy can be computed  for all $x,x'$ as
\[
 Q(x'|x) = \frac{P(x'|x) z(x')}{\sum_y P(y|x) z(y)}.
\]

\subsection{The convex optimization view}\label{sec:convex}
We also provide an alternative (and, to our knowledge, yet unpublished) view of the optimal control problem in LMDPs, based on convex 
optimization. For the purposes of this paper, we find this form to be more insightful, as it enables us to study our 
learning problem in the framework of online convex optimization \citep{Haz11,Haz16,SS12}. To derive this form, observe that under 
our assumptions, every feasible policy $Q$ induces a stationary distribution $\mu_Q$ over the state space $\X$ satisfying $\mu_Q\transpose 
= \mu_Q\transpose Q$. This stationary distribution and the policy together induce a distribution $\pi_Q$ over $\X^2$ defined for all $x,x'$ 
as $\pi_Q(x,x') = \mu_Q(x) Q(x'|x)$. We will call $\pi_Q$ as the \emph{stationary transition measure} induced by $Q$, which is motivated by 
the observation that $\pi_Q(x,x')$ corresponds to the probability of observing the transition $x\ra x'$ in the equilibrium state: 
$\pi_Q(x,x') = \lim_{T\ra \infty} \frac 1T \sum_{t=1}^T \PP{X_t = x, X_{t+1}=x'}$. Notice that, with this notation, the average 
cost-per-stage of policy $Q$ can be rewritten in the form
\[
\begin{split}
 &\lim_{T\ra\infty} \frac 1T \sum_{t=1}^T \EE{\loss(X_t,Q)} = \sum_{x} \mu_Q(x) \Bpa{c(x) + \kl{Q(\cdot|x)}{P(\cdot|x)}}
 \\
 &\qquad\qquad= \sum_{x,x'} \pi_Q(x,x') \pa{c(x) + \log\frac{\pi_Q(x,x')}{P(x'|x)\sum_y \pi_Q(x,y)}}
 \\
 &\qquad\qquad= \sum_{x,x'} \pi_Q(x,x') \log\frac{\pi_Q(x,x')}{\sum_y \pi_Q(x,y)} + \sum_{x,x'} \pi_Q(x,x') \pa{c(x) - \log\pa{P(x'|x)}}.
\end{split}
\]
The first term in the final expression above is the \emph{negative conditional entropy} of $X'$ relative to $X$, where $(X,X')$ is a pair 
of random states drawn from $\pi_Q$. Since the negative conditional entropy is convex in $\pi_Q$ (for a proof, see 
Appendix~\ref{app:convexity}) and the second term in the expression is linear in $\pi_Q$, we can see that $\lambda$ is a convex function of 
$\pi_Q$. This suggests that we can 
view the optimal control problem as having to find a feasible stationary transition measure $\pi$ that minimizes the expected costs. 
In short, defining 
\begin{equation}\label{eq:statcost}
 f(\pi;c) = \sum_{x,x'} \pi(x,x') \pa{c(x) + \log\frac{\pi(x,x')}{P(x'|x)\sum_y \pi(x,y)}}
\end{equation}
and $\Delta(M)$ as the (convex) set of feasible stationary transition measures $\pi$ satisfying
\begin{equation}\label{eq:statdist}
\begin{split}
 \sum_{x'} \pi(x,x') &= \sum_{x''} \pi(x'',x) \;\;\;\;\;\,\qquad (\forall x),
 \\
 \sum_{x,x'} \pi(x,x')&=1,
 \\
 \pi(x,x') &\ge 0 \qquad\qquad\qquad\qquad (\forall x,x'),
 \\
 \pi(x,x') &= 0 \qquad\qquad\qquad\qquad (\forall x,x': P(x'|x)=0),
\end{split}
\end{equation}
the optimization problem can be succinctly written as $\min_{\pi\in\Delta(M)} f(\pi;c)$. In Appendix~\ref{app:dual}, we provide a 
derivation of the optimal control given by Equation~\eqref{eq:bellman_Z} starting from the formulation given above. We also remark that our 
analysis will heavily rely on the fact that $f(\pi;c)$ is affine in $c$.

\section{Online learning in linearly solvable MDPs}\label{sec:online}
We now present the precise learning setting that we consider in the present paper.
We will study an online learning scheme where for each round $t=1,2,\dots,T$, the following steps are repeated:
\begin{enumerate}[leftmargin=.7cm]
 \item The learner chooses a transition kernel $Q_t:\X\ra \Delta(\X)$ satisfying $\mathop{supp}Q_t(\cdot|x) \subseteq 
\mathop{supp}P(\cdot|x)$ for all $x\in\X$.
 \item The learner observes $X_t\in\X$ and draws the next state $X_{t+1}\sim Q_t(\cdot|X_t)$.
 \item Obliviously to the learner's choice, the environment chooses state-cost function $c_t:\X\ra[0,1]$.
 \item The learner incurs the cost
 \[
  \ell_t(X_t,Q_t) = c_t(X_t) + \kl{Q_t(\cdot|X_t)}{P(\cdot|X_t)}.
 \]
 \item The environment reveals the state-cost function $c_t$.
\end{enumerate}
The key change from the stationary setting described in the previous section is that the state-cost function now may \emph{change 
arbitrarily} between each round, and the learner is only allowed to observe the costs \emph{after it has made its decision}. We stress 
that we assume that the learner \emph{fully knows} the passive dynamics, so the only difficulty comes from having to deal with the 
changing costs. As usual in the online-learning literature, our goal is to do nearly as well as the best \emph{stationary} policy 
chosen in hindsight after observing the entire sequence of cost functions. To define our precise performance measure, we first define the 
average reward of a policy $Q$ as
\[
 \Loss_T(Q) = \EE{\sum_{t=1}^T \loss_t(X_t',Q)},
\]
where the state trajectory $X_t'$ is generated sequentially as $X_t'\sim Q(\cdot|X_{t-1}')$ and the expectation integrates over the 
randomness of the 
transitions. Having this definition in place, we can specify the best stationary policy\footnote{The existence of the minimum is warranted 
by the fact that $\Loss_T$ is a continuous function bounded from below on its compact domain.} $Q^*_T = \argmin_{Q} \Loss_T(Q)$ and define 
our 
performance measure as the (total expected) \emph{regret} against $Q^*_T$:
\[
 R_T = \EE{\sum_{t=1}^T \ell_t(X_t,Q_t)} - \Loss_T(Q^*_T),
\]
where the expectation integrates over both the randomness of the state transitions and the potential randomization used by the 
learning algorithm. Having access to this definition, we can now formally define the goal of the learner as having to come up 
with a sequence of policies $Q_1,Q_2,...$ that guarantee that the total regret grows sublinearly, that is, that the average per-round 
regret asymptotically converges to zero.

For our analysis, it will be useful to define an idealized version of the above online optimization problem, where the learner is 
allowed to \emph{immediately switch} between the stationary distributions of the chosen policies. By making use of the convex-optimization 
view given in Section~\ref{sec:convex}, we define an auxiliary online convex optimization (or, in short, OCO, see, e.g., 
\citealp{Haz11,SS12}) problem called the 
\emph{idealized OCO problem} where in each round $t$, the following steps are repeated:
\begin{enumerate}
 \item The learner chooses the stationary transition measure $\pi_t\in\Delta(M)$.
 \item Obliviously to the learner's choice, the environment chooses the loss function $\tloss_t = f(\cdot;c_t)$.
 \item The learner incurs a loss of $\tloss_t(\pi_t)$. 
 \item The environment reveals the loss function $\tloss_t$.
\end{enumerate}
The performance of the learner 
in this setting is measured by the \emph{idealized regret}
\[
 \overline{R}_T = \sum_{t=1}^T \tloss_t(\pi_t) - \min_{\pi\in\Delta(M)}\sum_{t=1}^T \tloss_t(\pi).
\]

Throughout the paper, we will consider \emph{oblivious environments} that choose the sequence of state-cost functions without taking into 
account the states visited by the learner. This assumption will enable us to simultaneously reason about the expected costs under any 
sequence of state distributions, and thus to make a connection between the idealized regret $\overline{R}_T$ and the true regret $R_T$. 
This technique was first used by \citet{even-dar09OnlineMDP} and was shown to be essentially inevitable by \citet{yu09ArbitraryRewards}: As 
discussed in their Section~3.1, no learning algorithm can avoid linear regret if the environment is not oblivious.

\section{Algorithm and main result}\label{sec:ftl}
In this section, we propose a simple algorithm for online learning in LMDPs based on the ``follow-the-leader'' (FTL) strategy. On a high 
level, the idea of this algorithm is greedily betting on the policy that seems to have been optimal for the total costs observed so 
far. While this strategy is known to fail catastrophically in several simple learning problems (see, e.g., \citealt{CBLu06:Book}), it is 
known to perform well in several important scenarios such as sequential prediction under the logarithmic loss \citep{MF92} or prediction 
with expert advice under bounded losses, given that losses are stationary \citep{Kot16} and often serves as a strong benchmark 
strategy \citep{REGK14,SNL14}. In our learning problem, following the leader is a very natural choice of algorithm, as the convex 
formulation of Section~\ref{sec:convex} suggests that we can effectively build on the analysis of Follow-the-Regularized-Leader-type 
algorithms without having to explicitly regularize the objective.

In precise terms, our algorithm computes the sequence of policies $Q_1,Q_2,\dots,Q_T$ by running FTL \emph{in the idealized setting}: in 
round $t$, the algorithm chooses the stationary transition measure 
\[
\begin{split}
 \pi_t &= \argmin_{\pi\in\Delta(M)} \sum_{s=1}^{t-1} \tloss_s(\pi) = \argmin_{\pi\in\Delta(M)} \sum_{s=1}^{t-1} f(\pi;c_s)
 \\
 &= \argmin_{\pi\in\Delta(M)} (t-1)\cdot f\pa{\pi;\frac{1}{t-1}\sum_{s=1}^{t-1}c_s} = \argmin_{\pi\in\Delta(M)} 
f\pa{\pi;\bc_t},
\end{split}
\]
where the third equality uses the fact that $f$ is affine in its second argument and the last step introduces the average state-cost 
function $\bc_t = \frac {1}{t-1} \sum_{s=1}^{t-1} c_s$. This form implies that $\pi_t$ can be computed as the optimal control for the 
state-cost function $\bc_t$, which can be done by following the procedure described in Section~\ref{sec:bellman}. Precisely, we define the  
diagonal matrix $G_t$ with its $i$\th~diagonal element $e^{-\bc_t(i)}$, let $\gamma_t$ be the largest eigenvalue of $G_tP$ and $z_t$ be the 
corresponding (unit-norm) right eigenvector. Also, let $v_t = -\log z_t$ and $\lambda_t = -\log\gamma_t$, and note that $\lambda_t = 
f(\pi_t;\bc_t)$ is the optimal average-cost-per-stage of $\pi_t$ given the cost function $\bc_t$.
Finally, we define the policy used in round $t$ as
\begin{equation}\label{eq:optQ}
 Q_t(x'|x) = \frac{P(x'|x) z_t(x')}{\sum_y P(y|x) z_t(y)}
\end{equation}
for all $x'$ and $x$. We denote the induced stationary distribution by $\mu_t$. The algorithm is presented as Algorithm~\ref{alg:ftl}.

\begin{algorithm}
 \textbf{Input:} Passive dynamics $P$.
 \\\textbf{Initialization:} $\bc_1(x) = 0$ for all $x\in\X$.
 \\\textbf{For $t=1,2,\dots,T$, repeat}
 \begin{enumerate}[leftmargin=.7cm]
  \item Construct $G_t = \left[\mbox{diag}(e^{-\bc_t})\right]$.
  \item Find the right eigenvector $z_t$ of $G_t P$ corresponding to the largest eigenvalue.
  \item Compute the policy
  \[
 Q_t(x'|x) = \frac{P(x'|x) z_t(x')}{\sum_y P(y|x) z_t(y)}.
\]
  \item Observe state $X_{t}$ and draw $X_{t+1}\sim Q_t(\cdot|X_{t})$.
  \item Observe state-cost function $c_t$ and update $\bc_{t+1} = \frac{\pa{t-1}\bc_t + c_t}{t}$.
 \end{enumerate}
  \caption{Follow The Leader in LMDPs} \label{alg:ftl}
\end{algorithm}

Now we present our main result. First, we state two key assumptions about the underlying passive dynamics; both of these assumptions are  
also made by \citet{GRW12}.
\begin{assumption}\label{ass:irred}
 The passive dynamics $P$ is irreducible and aperiodic. In particular, there exists a natural number $H>0$ such that $\pa{P^n}(y|x)>0$ for 
all $n\ge H$ and all $x,y\in\X$. We will refer to $H$ as the (worst-case) \emph{hitting time}.
\end{assumption}
\begin{assumption}\label{ass:ergod}
 The passive dynamics $P$ is ergodic in the sense that its Markov--Dobrushin ergodicity coefficient is strictly less than $1$:
 \[
  \alpha(P) = \max_{x,y\in\X} \onenorm{P(\cdot|x)-P(\cdot|y)} < 1.
 \]
\end{assumption}
A standard consequence (see, e.g., \citealt{Sen2006}) of Assumption~\ref{ass:ergod} is that the passive dynamics mixes quickly: for any 
distributions $\mu,\mu'\in\Delta(\X)$, we have
\[
 \onenorm{\pa{\mu-\mu'}\transpose P}\le \alpha(P)\onenorm{\mu-\mu'}.
\]
We will sometimes refer to $\tau(P) = \pa{\log\pa{1/\alpha(P)}}^{-1}$ as the \emph{mixing time} associated with $P$. 
Now we are ready to state our main result:
\begin{theorem}\label{thm:main}
 Suppose that the passive dynamics satisfies Assumptions~\ref{ass:irred} and~\ref{ass:ergod}. Then, the regret of Algorithm~\ref{alg:ftl} 
satisfies $R_T = O(\log^2 T)$.
\end{theorem}
The asymptotic notation used in the theorem hides a number of factors that depend only on the 
passive dynamics $P$. In particular, the bound scales polynomially with the worst-case mixing time $\tau$ of 
any optimal policy, and shows no \emph{explicit} dependence on the number of states.\footnote{Of course, the mixing time time does depend 
on the size of the state space in general.} We explicitly state the bound at the end of the proof 
presented in the next section as Equation~\eqref{eq:fullbound}, when all terms are formally defined.

\section{Analysis}\label{sec:analysis}
In this section, we provide a series of lemmas paving the way towards proving Theorem~\ref{thm:main}. The attentive reader may find some of 
these lemmas familiar from related work: indeed, we build on several technical results from \citet{even-dar09OnlineMDP,neu14o-mdp-full} and 
\citet{GRW12}. Our main technical contribution is an efficient combination of these tools that enables us to go way beyond the 
best known bounds for our problem, proved by \citet{GRW12}. Throughout the section, we will assume that the conditions of 
Theorem~\ref{thm:main} are satisfied.
 
Before diving into the analysis, we state some technical results that we will use several times. We defer all proofs to 
Appendix~\ref{sec:app}. First, we present some important facts regarding LMDPs with bounded state-costs. In particular, we define $Q^*(c)$ 
as the optimal policy with respect to an arbitrary state-cost function $c$ and let $\mathcal{C}$ be the set of all state-costs bounded 
in $[0,1]$. We define $\mathcal{Q}^*$ as the set of optimal policies induced by state-cost functions in $\mathcal{C}$: $\mathcal{Q}^* = 
Q^*(\mathcal{C})$. Observe that $Q_t\in\mathcal{Q}^*$ for all $t$, as $Q_t = Q^*(\bc_t)$ and $\bc_t \in \mathcal{C}$ for all $t$. Below, we 
give several useful results concerning policies in $\Qstar$.
For stating these results, let $c\in\mathcal{C}$ and $Q = Q^*(c)$. We first note that the average cost $\lambda$ of $Q$ is bounded in 
$[0,1]$: By the Perron-Frobenius theorem (see, e.g., \citealp[Chapter~8]{Mey00}), we have that the largest eigenvalue of $GP$ is 
bounded by the maximal and minimal row sums of $GP$: $e^{-\lambda}\in[e^{-\max_x c(x)},e^{-c(x)}]$, which translates to 
having $\lambda\in[0,1]$ under our assumptions. The next key result bounds the value functions and the control costs in terms of the 
hitting time:
\begin{lemma}\label{lem:vbound}
 For all $x,y$ and $t$, the value functions satisfy $v_t(x)-v_t(y) \le H$. Furthermore, all policies $Q\in\Qstar$ satisfy
\[
 \max_x \kl{Q(\cdot|x)}{P(\cdot|x)} \le H+1.
\]
\end{lemma}
The proof is loosely based on ideas from \citet{bartlett09regal}.
The second statement guarantees that the mixing time $\tau(Q) = \pa{\log(1/\alpha(Q))}^{-1}$ is finite for all policies in $\Qstar$:
\begin{lemma}\label{lem:mixing}
 The Markov--Dobrushin coefficient $\alpha(Q)$ of any policy $Q\in\Qstar$ is bounded as
 \[
  \alpha(Q) \le \alpha(P) + \pa{1-\alpha(P)} \pa{1-e^{-H-2}}<1.
 \]
\end{lemma}
The proof builds on the previous lemma and uses standard ideas from Markov-chain theory.
In what follows, we will use $\tau = \max_{Q\in\Qstar} \tau(Q)$ and $\alpha = \max_{Q\in\Qstar} \alpha(Q)$ to denote the worst-case 
mixing time and ergodicity coefficient, respectively. With this notation, we can state the following lemma that establishes that the value 
functions are $2\tau$-Lipschitz with respect to the state-cost function. For pronouncing and proving the statement, it is useful to define 
the \emph{span seminorm} $\spannorm{c} = \max_x c(x) - \min_y c(y)$. Note 
that it is easy to show that $\spannorm{\cdot}$ is indeed a seminorm as it satisfies all the requirements to be a norm except that it maps 
all constant vectors (and not just zero) to zero. 
\begin{lemma}\label{lem:c_to_v}
 Let $f$ and $g$ be two state-cost functions taking values in the interval $[0,1]$ and let $v_f$ and $v_g$ be the corresponding optimal 
value 
functions. Then,
\[
 \spannorm{v_f - v_g} \le 2\tau \infnorm{f-g}.
\]
\end{lemma}
 The proof roughly follows the proof of Proposition~3 of \citet{GRW12}, with the slight difference that we make the constant factor in the 
bound explicit. A consequence of this result is our final key lemma in this section that actually makes our fast rates possible: a bound on 
the change-rate of the policies chosen by the algorithm.
\begin{lemma}\label{lem:change} 
$\max_x \onenorm{Q_t(\cdot|x) - Q_{t+1}(\cdot|x)} \le \frac{\tau}{t}$.
\end{lemma}
The proof is based on ideas by \citet{GRW12}. As for the proof of Theorem~\ref{thm:main}, we follow the path of 
\citet{even-dar09OnlineMDP,neu14o-mdp-full,GRW12}, and first analyze the idealized setting where the learner is allowed to directly pick 
stationary distributions instead of policies. Then, we show how to relate the idealized regret of FTL to its true regret in the original 
problem.

\subsection{Regret in the idealized OCO problem}
Let us now consider the idealized online convex optimization problem described at the end of Section~\ref{sec:online}. 
In this setting, our algorithm can be formally stated as choosing the stationary transition measure $\pi_t = \argmin_{\pi\in\Delta(M)} 
f(\pi;\bc_t)$. This view enables us to follow a standard proof technique for analyzing online convex optimization algorithms, going back to 
at least \citet{MF92}.
The first ingredient of our proof is the so-called ``follow-the-leader/be-the-leader'' lemma \citet[Lemma~3.1]{CBLu06:Book}:
\begin{lemma} \label{lem:btl}
$\sum_{t=1}^T \tloss_t(\pi_{t+1}) \le \min_\pi \sum_{t=1}^T \tloss_t(\pi)$.
\end{lemma}
The second step exploits the bound on the change rate of the policies to show that looking one step into the future does not buy much 
advantage. Note however that controlling the change rate is not sufficient by itself, as our loss functions are effectively unbounded.
\begin{lemma}\label{lem:price}
$\sum_{t=1}^T \pa{\tloss_t(\pi_{t}) - \tloss_t(\pi_{t+1})} \le 2 \pa{\tau^2+1} \pa{1+\log T}$.
\end{lemma}
In the interest of space, we only provide a proof sketch here and defer the full proof to Appendix~\ref{app:price}.
\begin{proofsketch}
 Let us define $\Delta_t = \bc_{t+1} - \bc_{t}$. By exploiting the affinity of $f$ in its second argument, we can start by proving
$\lambda_t - \lambda_{t+1} \le \infnorm{\Delta_t}$. Furthermore, by using the form of the optimal policy $Q_t$ given in Eq.~\eqref{eq:optQ} 
and the form of $f$ given in Eq.~\eqref{eq:statcost}, we can obtain
\begin{align*}
 \tloss_t(\pi_{t}) - \tloss_t(\pi_{t+1}) &= \pa{\mu_{t} - \mu_{t+1}}\transpose \pa{c_t + \bc_{t}} + \mu_{t+1}\transpose \pa{\bc_{t} - 
\bc_{t+1}}
 + \lambda_{t} - \lambda_{t+1}
 \\
 &\le 2 \onenorm{\mu_{t+1} - \mu_t} + 2\infnorm{\Delta_t}.
\end{align*}
The first term can be bounded by a simple argument (see, e.g., Lemma~4 of \citealt{neu14o-mdp-full}) that leads to
\[
\onenorm{\mu_{t+1} - \mu_t} \le \max\ev{\tau(Q_t),\tau(Q_{t+1})} \max_x \onenorm{Q_{t+1}(\cdot|x)-Q_t(\cdot|x)}.
\]
Now, the first factor can be bounded by $\tau$ and the second by appealing to Lemma~\ref{lem:change}. The proof is 
concluded by plugging the above bounds into Equation~\eqref{eq:lpidiff}, using $\infnorm{\Delta_t} \le 1/t$, summing up both sides, and 
noting that $\sum_{t=1}^T 1/t \le 1 + \log T$.
\end{proofsketch}
Putting Lemmas~\ref{lem:btl} and~\ref{lem:price} together, we obtain the following bound on the idealized regret of FTL:
\begin{lemma}\label{lem:ideal}
$\overline{R}_T \le 2\pa{\tau^2+1}\pa{1+\log T}$.
\end{lemma}

\subsection{Regret in the reactive setting}
We first show that the advantage of the true best policy $Q^*_T$ over our final policy $Q_{T+1}$ is bounded. 
\begin{lemma} Let $p^* = \min_{x,x':P(x'|x)>0} P(x'|x)$ be the smallest non-zero transition probability under the passive dynamics and $B = 
-\log p^*$. Then,
$ \sum_{t=1}^T \bloss_t(\pi_{T+1}) - \Loss_T(Q_T^*) \le \pa{2\tau + 2}\pa{B+1}$.
\end{lemma}
The proof follows from applying Lemma~1 from \citet{neu14o-mdp-full} and observing that $\loss_t(X_t,Q_T^*) \le B+1$ holds for all $t$.
It remains to relate the total cost of FTL to the total idealized cost of the algorithm. This is done in the following lemma:
\begin{lemma} \label{lem:trueloss}
$ \sum_{t=1}^T \pa{\EE{\loss_t(Q_t,X_t)} - \bloss_t(\pi_{t})} \le \pa{\tau+1}^3 \pa{1+\log T}^2 + 2\pa{\tau + 1}\pa{3+\log T}$.
\end{lemma}
\begin{proof}
  Let $p_t(x) = \PP{X_t = x}$. Similarly to the proof of Lemma~\ref{lem:price}, we rewrite $\bloss_t(\pi_t)$ using 
Equation~\eqref{eq:lpiform} to obtain
 \[
 \begin{split}
  \EE{\loss_t(Q_t,X_t) - \bloss_t(\pi_{t})}
  &=
  \sum_{x} \pa{p_t(x) - \mu_t(x)} \pa{c_t(x) + v_t(x) + \lambda_{t} - \bc_{t}(x) - \sum_{x'} Q_t(x'|x) v_t(x')}
  \\
  &\le  \sum_{x} p_t(x) \pa{v_t(x) - \sum_{x'} Q_t(x'|x) v_{t}(x')} + \onenorm{p_t- \mu_t},
 \end{split}
 \]
 where the last step uses $\sum_{x} \mu_t(x) Q_t(x'|x) = \mu_t(x')$ and $\infnorm{c_t - \bc_t}\le 1$.
 Now, noticing that $\sum_{x} p_t(x) Q_t(x'|x) = p_{t+1}(x')$, we obtain
 \[
 \begin{split}
  &\sum_{t=1}^T \EE{\loss_t(Q_t,X_t) - \bloss_t(\pi_{t})}
  \le \sum_{t=1}^T \pa{p_t-p_{t+1}}\transpose v_t + \sum_{t=1}^T \onenorm{\mu_t - p_t}
  \\
  &\quad= \sum_{t=1}^{T} p_t\transpose \pa{v_t - v_{t-1}} + \sum_{t=1}^T \onenorm{\mu_t - p_t} - p_{T+1}\transpose v_T
  \le \sum_{t=1}^T \frac{2\tau}{t} + \sum_{t=1}^T \onenorm{\mu_t - p_t} - p_{T+1}\transpose v_T,
 \end{split}
 \]
 where the last inequality uses Lemma~\ref{lem:c_to_v} and $\infnorm{\bc_{t} - \bc_{t-1}}\le 1/t$ to bound the first term. By 
Lemma~\ref{lem:c_to_v}, this last term can be bounded by $\norm{v_T}_s = \norm{v_T - v_0}_s\le 2\tau\infnorm{\bc_T} \le 2\tau$, where $v_0$ 
is the value function corresponding to the all-zero state-cost function.
 
 In the rest of the proof, we are going to prove the inequality
 \begin{equation}\label{eq:pmudiff}
 \onenorm{\mu_t - p_t} \le 2 e^{-\pa{t-1}/\tau} +  \frac{2(\tau+1)^3\pa{1+ \log t}}{t}.
 \end{equation}
 It is easy to see that this trivially holds for $\pa{2 \tau \log t}/t \ge 
1$, so we will assume that the contrary holds in the following derivations.
To prove Equation~\eqref{eq:pmudiff} for larger values of $t$, we can follow the proofs 
of Lemma~5 of \citet{neu14o-mdp-full} or Lemma~5.2 of \citet{even-dar09OnlineMDP} to obtain
\begin{equation}\label{eq:mu_to_p}
 \onenorm{\mu_t - p_t} \le 2 e^{-\pa{t-1}/\tau} +  \tau \pa{\tau+1} \sum_{n=1}^{t-1} \frac{e^{-(t-n)/\tau}}{n}.
\end{equation}
For completeness, we include a proof in Appendix~\ref{app:mu_to_p}.
For bounding the last term, we split the sum at $B=\left\lfloor t - \tau\log t\right\rfloor$:
\[
 \begin{split}
  \sum_{n=1}^{t-1} \frac{e^{-(t-n)/\tau}}{n} &= 
  \sum_{n=1}^{B} \frac{e^{-(t-n)/\tau}}{n} + \sum_{n=B+1}^{t-1} \frac{e^{-(t-n)/\tau}}{n}
 \\
 &= e^{-(t-B)/\tau} \sum_{n=1}^{B} \frac{e^{-(B-n)/\tau}}{n} + \sum_{n=B+1}^{t} \frac{e^{-(t-n)/\tau}}{n}
 \\
 &\le \frac{1}{t} \cdot \frac{1}{1-e^{-1/\tau}} + \frac{\tau\log t}{t-\tau\log t} \le 
 \frac{\tau }{t} + \frac{\tau\log t}{t} \cdot \frac{1}{1-\pa{\tau\log t}/ t}
 \\
 &\le \frac{\tau}{t} + \frac{2\tau \log t}{t} \le \frac{2\tau\pa{1 + \log t}}{t},
 \end{split}
\]
where the first inequality follows from bounding the $1/n$ factors by $1$ and $1/B$, respectively, and bounding the sums by 
the full geometric sums. The second-to-last inequality follows from our assumption that $(2\tau\log t)/t \le 1$.  That is, we have 
successfully proved Equation~\eqref{eq:pmudiff}. Now the statement of the lemma follows from summing up for all $t$ and noting that 
$\sum_{t=1}^T \frac{1}{t} \le 1 + \log T$ and $\sum_{t=1}^T e^{-\pa{t-1}/\tau} \le \tau + 1$.
\end{proof}
Now the proof of Theorem~\ref{thm:main} follows easily from combining the bounds of Lemmas~\ref{lem:ideal}--\ref{lem:trueloss}. The result 
is
\begin{equation}\label{eq:fullbound}
 R_T \le 
 2\pa{\tau+1}^3 \pa{1+\log T}^2 + 2\pa{\tau^2 + \tau + 2}  (3+\log T) + \pa{2\tau+2}\pa{B+2}.
\end{equation}
Thus, we can see that the bound indeed demonstrates a polynomial dependence on the mixing time $\tau$, 
and depends logarithmically on the smallest non-zero transition probability $p^*$ via $B = -\log p^*$.

\section{Discussion}\label{sec:discussion}
In this paper, we have shown that, besides the well-established computational advantages, linearly solvable MDPs also admit a remarkable
information-theoretic advantage: fast learnability in the online setting. In particular, we show that achieving a regret of $O(\log^2 T)$ 
is achievable by the simple algorithm of following the leader, thus greatly improving on the best previously known regret bounds of 
$O(T^{3/4})$. At first sight, our improvement may appear dramatic: in their paper, \citet{GRW12} pose the possibility of improving their 
bounds to $O(\sqrt{T})$ as an important open question (Sec.~VII.). In light of our results, these conjectured improvements are also grossly 
suboptimal. On the other hand, our new results can be also seen to complement well-known results on fast rates in online learning (see, 
e.g., \citealt{EGMRW15} for an excellent summary). Indeed, our learning setting can be seen as a generalized variant of sequential 
prediction under the  relative-entropy loss (see, e.g., \citealp[Sec.~3.6]{CBLu06:Book}), which is known to be \emph{exp-concave}. Such 
exp-concave losses are well-studied in the online learning literature, and are known to allow logarithmic regret bounds \citep{KW99,HAK07}.

Inspired by these related results, we ask the question: Is the loss function $f$ defined in Section~\ref{sec:convex} exp-concave? While our 
derivations Appendix~\ref{app:convexity} indicate that $f$ has curvature in certain directions, we were not able to prove its 
exp-concavity. Similarly to the approach of \citet{MF92}, our analysis in the current paper merely exploits the Lipschitzness of the optimal 
policies with respect to the cost functions, but otherwise does not explicitly make use of the curvature of $f$. We hope that our work 
presented in this paper will inspire future studies that will clarify the exact role of the LMDP structure in efficient online learnability, 
potentially also leading to a better understanding of policy gradient algorithms for LMDPs \citep{Tod10}.

Finally, let us comment on the tightness of our bounds. Regardless of whether the loss function $f$ is exp-concave or not, 
we are almost certain that our rates can be improved to at least $O(\log T)$ by using a more sophisticated algorithm. While our focus in 
this paper was on improving the asymptotic regret guarantees, we also slightly improve on the results \citet{GRW12} in that we make the 
leading constants more explicit. However, we expect that the dependence on these constants may also be improved in future work. 
Note however that the potential 
looseness of our bounds does not impact the performance of the algorithm itself, as it never makes use of any problem-dependent constants.

\paragraph{Acknowledgements}
This work was supported by the UPFellows Fellowship (Marie Curie COFUND program n${^\circ}$ 
600387) and the Ramon y Cajal program RYC-2015-18878 (AEI/MINEICO/FSE, UE). 
The authors wish to thank the three anonymous reviewers for their valuable comments that helped to improve the paper.

\newpage
\appendix

\section{The convex optimization view of optimal control in LDMPs}
This section summarizes some facts regarding the convex optimization formulation of Section~\ref{sec:convex}. We first show that the 
negative conditional entropy constituting the only nonlinear term in the objective $f(\pi;c)$ is convex.
\subsection{The convexity of the negative conditional entropy}
\label{app:convexity}
Let us consider the joint probability distribution $\pi$ on the finite set $\X^2$. We denote $\mu(x) = \sum_y p(x,y)$ and $Q(y|x) = 
p(x,y)/\mu$. We study the negative conditional entropy of 
$(X,Y)\sim\pi$ as a function of $\pi$:
\[
\begin{split}
R(\pi) 
 &= \sum_{x,y} \pi(x,y) \log\frac{\pi(x,y)}{\sum_{y'} \pi(x,y')}
= \sum_{x,y} \pi(x,y) \log\frac{\pi(x,y)}{\mu(x)} 
\end{split}
\]
We will study the Bregman divergence $B_R$ corresponding to $R$:
\[
 B_R\left(\left.\pi'\right|\pi\right) = R(\pi') - R(\pi) - \nabla R(\pi) \transpose (\pi' - \pi).
\]
Our aim is to show that $B_R$ is nonnegative, which will imply the convexity of $R$.

We begin by computing the partial derivative of $R(\pi)$ with respect to $\pi(x,y)$:
\[
 \frac{\partial R(\pi)}{\partial \pi(x,y)} = \log\pa{\pi(x,y)} - \log\pa{\mu(x)},
\]
where we used the fact that $\frac{\partial \mu(x)}{\partial \pi(x,y)} = 1$ for all $y$.
With this expression, we have
$-H(Y|X)$:
\[
\begin{split}
 R(\pi) + \nabla R(\pi) \transpose (\pi' - \pi)
 &= \sum_{x,y} \pi(x,y) \log\frac{\pi(x,y)}{\mu(x)} + \sum_{x,y} \pa{\pi'(x,y) - \pi(x,y)} \log\frac{\pi(x,y)}{\mu(x)}
 \\
 &= \sum_{x,y} \pi'(x,y) \log\frac{\pi(x,y)}{\mu(x)}.
\end{split}
\]
Thus, the Bregman divergence takes the form
\[
\begin{split}
 B_R\left(\left.\pi'\right|\pi\right) &= \sum_{x,y} \pi'(x,y) \pa{\log\frac{\pi'(x,y)}{\mu'(x)} - \log\frac{\pi(x,y)}{\mu(x)}}
 \\
 &= \sum_{x,y} \pi'(x,y) \log\frac{Q'(y|x)}{Q(y|x)}
 = \sum_{x} \mu'(x) \sum_y Q'(y|x) \log\frac{Q'(y|x)}{Q(y|x)}
 \\
 &= \sum_{x} \mu'(x) \kl{Q'(\cdot|x)}{Q(\cdot|x)} \ge \frac 12 \sum_{x} \mu'(x) \onenorm{Q'(\cdot|x) - Q(\cdot|x)}^2,
\end{split}
\]
where the last step follows from Pinsker's inequality. Thus, we have shown that the Bregman divergence $B_R$ is nonnegative on 
$\Delta(\X^2)$, proving that $R(\pi)$ is convex.

\subsection{Derivation of the optimal control}
\label{app:dual}
Here, we give an alternative derivation of the optimal control given in Section~\ref{sec:bellman} based on the optimization problem 
$\min_{\pi\in\Delta(M)} f(\mu;c)$ for an arbitrary bounded state-cost function $c$. As a reminder, $f(\pi;c)$ is given by
\[
 f(\pi;c) = \sum_{x,x'} \pi(x,x') \pa{c(x) + \log\frac{\pi(x,x')}{P(x'|x)\sum_y \pi(x,y)}}
\]
and the feasible set $\Delta(M)$ is given by the following convex constraints:
\[
\begin{split}
 \sum_{x'} \pi(x,x') &= \sum_{x''} \pi(x'',x) \;\;\;\;\;\,\qquad (\forall x),
 \\
 \sum_{x,x'} \pi(x,x')&=1,
 \\
 \pi(x,x') &\ge 0 \qquad\qquad\qquad\qquad (\forall x,x'),
 \\
 \pi(x,x') &= 0 \qquad\qquad\qquad\qquad (\forall x,x': P(x'|x)=0).
\end{split}
\]
We begin by slightly adjusting the definition of $f(\cdot;c)$ for it to become a \emph{barrier function}: we set $f(\pi;c) = \infty$ for 
all $\pi$ not satisfying the last two constraints. It is easy to see that this adjustment does not change the optimum of $f(\cdot;c)$, but 
it helps getting rid of the inequality constraints. Thus, with this form of $f$, we can characterize the optimum of 
$f(\cdot;c)$ using the technique of Lagrange multipliers\footnote{Alternatively, one could introduce KKT multipliers for all constraints 
and eliminate the last two by complementary slackness, which yields the same characterization.}.

Precisely, we introduce a Lagrange multiplier $v(x)$ for every $x$ to enforce the first constraint and $\lambda$ to 
enforce the second one, and write 
the Lagrangian as
\[
\begin{split}
 \mathcal{L}(\pi;v,\lambda) =& \sum_{x,x'} \pi(x,x') \pa{c(x) + \log\frac{\pi(x,x')}{P(x'|x)\sum_y\pi(x,y)}} + \lambda \pa{\sum_{x,x'} 
\pi(x,x') - 1}
 \\
 &+\sum_{x,x'} v(x) \pa{\pi(x,x') - \pi(x',x)}
 \\
 =& \sum_{x,x'} \pi(x,x') {\log\frac{\pi(x,x')}{P(x'|x)\sum_y\pi(x,y)}} + \sum_{x,x'} \pi(x,x') \pa{c(x) + \lambda + v(x) - v(x')} - 
\lambda
\end{split}
\]
Let $\mu(x) = \sum_y \pi(x,y)$, noting that $\partial \mu(x)/\partial \pi(x,y) = 1$ for all $y$. Differentiate the Lagrangian 
with respect to a fixed $\pi(x,x')$:
\[
\begin{split}
 \frac{\partial\mathcal{L}(\pi;v,\lambda)}{\partial \pi(x,x')}
=& \log \pa{\pi(x,x')} - \log\pa{\mu(x)} + \pa{c(x) + \lambda + v(x) - v(x') - \log P(x'|x)}.
\end{split}
\]
Setting the gradient to zero, we obtain the following formula for $\pi(x,x')/\mu(x)$:
\[
 \frac{\pi(x,x')}{\mu(x)} = P(x'|x) \cdot \exp\pa{-c(x) - \lambda - v(x) + v(x')}. 
\]
Since $\pi(x,y)\ge 0$ for all $x,y$, we have $\sum_{x'} \frac{\pi(x,x')}{\sum_y \pi(x,y)} = 1$ and thus
\[
 \sum_{x'} P(x'|x) \exp\pa{-c(x) - \lambda + v(x') - v(x)} = 1.
\]
Introducing the variables $z(x)$ for all $x$, we recover the linear system of equations in Equation~\eqref{eq:bellman_Z}:
\[
 z(x) = \sum_{x'} P(x'|x) \exp\pa{-c(x) - \lambda} z(x').
\]
Plugging back into the Lagrangian, we obtain that the dual function is $\mathcal{L}(\lambda) = -\lambda$, which now needs to be maximized 
subject to the above constraint, implying that $\exp(-\lambda)$ is indeed the largest eigenvalue of the matrix $GP$. Furthermore, by strong 
duality, the $\lambda$ maximizing the dual is indeed the minimum of the primal $f(\pi;c)$ on $\Delta(M)$.

\section{Technical proofs}\label{sec:app}
In this section, we prove the preliminary lemmas from Section~\ref{sec:analysis}.
\subsection{The proof of Lemma~\ref{lem:vbound}}
 The idea of the proof is similar to the proof of Theorem~4 of \citet{bartlett09regal}. By our Assumption~\ref{ass:irred} and the fact 
that all feasible control policies retain the structural properties of the passive dynamics, our MDP satisfies the conditions of 
Proposition~4.3.2 of \citet{Ber07:DPbookVol2}, so \emph{value iteration} converges to the solution of the Bellman optimality equations.
 Let $J_n(x)$ be the total expected cost of the best $n$-horizon policy started output by the value iteration procedure and let $q_n$ 
denote the corresponding policy. Then, consider the strategy of following the passive dynamics from $x$ until hitting $y$ and then 
switching 
to the optimal finite-horizon policy optimized for the remaining rounds. By the finite-horizon optimality of $q_n$, this strategy is 
clearly suboptimal: letting $L$ denote the \emph{random} number of steps taken for reaching $y$ under the passive dynamics, this 
suboptimality can be expressed as $J_n(x) \le \EE{L + J_{n-L}(y)}$.
Thus, by the fact that value iteration converges, we have
\begin{align*}
 v(x) - v(y) =& \lim_{n\ra\infty} \bpa{J_n(x) - J_n(y)}
 \\
 \le& \lim_{n\ra\infty} \bpa{\EE{L + J_{n-L}(y) - J_n(y)}} \le H,
\end{align*}
where we used that $J_{k}(y) \le J_{n}(y)$ for all $k\le n$ and that $\EE{L} \le H$ by Assumption~\ref{ass:irred}. This concludes the proof 
of the first statement.
 As for the second statement, note that 
the associated optimal policy $Q$ can be written as
\begin{equation}\label{eq:klform}
  Q(x'|x) = P(x'|x) \exp\pa{\lambda - c(x) - v(x') + v(x)},
\end{equation}
 so the control cost can be bounded for all $x$ as
 \[
  \kl{Q(\cdot|x)}{P(\cdot|x)} = v(x) + \lambda - c_t(x) - \sum_{x'} Q(x'|x) v(x') \le H+1,
 \]
 thus concluding the proof.\jmlrQED

\subsection{The proof of Lemma~\ref{lem:mixing}}\label{sec:app:mixing}
It is well known (see, e.g., \cite{Sen2006}) that the Markov-Dobrushin coefficient $\alpha(P)$ of a 
transition kernel $P$ satisfies
\[
 \alpha(P) = 1 - \min_{x,y\in\X} \sum_s \min\ev{P(s|x),P(s|y)}
\]
Now, let us observe that by the expression~\eqref{eq:klform}, we can prove the statement of the lemma as
\[
\begin{split}
 \alpha(Q) =&  1 - \min_{x,y\in\X} \sum_s \min\ev{Q(s|x),Q(s|y)}
 \\
 =& 1 - \min_{x,y\in\X} \sum_k \min\ev{P(s|x)\exp\pa{v(x) - c(x)},P(s|y)\exp\pa{v(y) - c(y)}} \exp\pa{\lambda - v(s)}
 \\
 \le & 1 - \min_{x,y\in\X} \sum_k \min\ev{P(s|x),P(s|y)} \exp\pa{\lambda - v(s) - 1 + \min_{s'} v(s')}
 \\
 \le & 1 - \min_{x,y\in\X} \sum_k \min\ev{P(s|x),P(s|y)} e^{-H-1} = e^{-1/\tau} + \pa{1-e^{-1/\tau}}\pa{1-e^{-H-1}},
\end{split}
\]
where the last inequality follows from applying Lemma~\ref{lem:vbound} to show $v(s) - v(s') \le H$ for this particular pair of states.
\jmlrQED

\subsection{The proof of Lemma~\ref{lem:c_to_v}}\label{sec:app:c_to_v}
 The proof roughly follows the proof of Proposition~3 of \citet{GRW12}. Let us define the Bellman optimality operator 
$B_c:\real^{\X}\ra\real^\X$ associated with the state-cost function $c$ that acts on any value function $v$ as $\pa{B_c v}(x) = c(x) + 
\min_{q\in\Delta(\X)} \ev{\kl{q}{P(\cdot|x)} + \sum_y q(y) v(y)}$. With this operator, we can express the Bellman optimality equations for 
$v_f$ as $v_f + \lambda_f = B_c v_f$. Thus, we have
\begin{equation}\label{eq:span}
\begin{split}
 \spannorm{v_f - v_g} &= \spannorm{B_f v_f - \lambda_f - B_g v_g + \lambda_g}
 \\
 &= \spannorm{B_f v_f - B_g v_g} \le \spannorm{B_f v_f - B_g v_f} + \spannorm{B_g v_f - B_g v_g},
\end{split}
\end{equation}
where the second equality follows from the fact that the span seminorm is insensitive to shifting by constants and the last step follows 
from the triangle inequality. The first term in the above expression can be easily bounded noting that $B_f v_f - B_g v_f = f - g$ by the 
definition of the Bellman operator. For the second term, we can follow the argument of \citet{GRW12} to show that
\[
 \spannorm{B_g v_f - B_g v_g} \le \frac 12 \spannorm{v_f - v_g} \max_{x,y} \onenorm{Q_f(\cdot|x) - Q_g(\cdot|y)}.
\]
Now for controlling the last factor in the above expression, we take an arbitrary pair of states $x$ and $y$ and write
\[
\begin{split}
 &\frac 12 \sum_s \left|Q_f(s|x) - Q_g(s|y)\right| = \frac 12 \sum_s \left(Q_f(s|x) + Q_g(s|y)\right) - \sum_s \min\ev{Q_f(s|x),Q_g(s|y)}
 \\
 &= 1 - \sum_s \min\ev{P(s|x)\exp(v_f(x)\!-\!f(x)\!+\!\lambda_f \!-\! v_f(s)),P(s|y)\exp(v_g(x)\!-\!g(x)\!+\!\lambda_f \!-\! v_g(s))}
 \\
 &\le 1 - \sum_s \min\ev{P(s|x),P(s|y)} e^{-H-1} = \alpha(P) + \pa{1-\alpha(P)}\pa{1-e^{-H-1}} = \alpha,
\end{split}
\]
where the first step uses the equality $|a+b| = \frac 12 |a+b| - \min\ev{a,b}$ that holds for any two real numbers $a,b$, the second step 
follows from Equation~\eqref{eq:klform}, the inequality from Lemma~\ref{lem:vbound}, and the last steps use the respective definitions of 
$\alpha(P)$ and $\alpha$.
In summary, we have proved that $\spannorm{B_g v_f - B_g v_g} \le \alpha \spannorm{v_f - v_g}$
holds for the worst-case ergodicity coefficient $\alpha<1$. Plugging this bound back into Equation~\eqref{eq:span} gives 
$\spannorm{v_f - v_g} \le \frac{1}{1-\alpha} \spannorm{f-g}$, which can be seen to imply the statement of the lemma after observing 
$\spannorm{f-g}\le 2\infnorm{f-g}$ and $\frac{1}{1-\alpha}\le \tau$.
\jmlrQED

\subsection{The proof of Lemma~\ref{lem:change}}\label{sec:app:v_to_Q}
The proof builds on the following lemma, adapted from Proposition~4 of \citet{GRW12}:
\begin{lemma}\label{lem:v_to_Q}
Let $f$ and $g$ be two state-cost functions taking values in the interval $[0,1]$, let $v_f$ and $v_g$ be the corresponding optimal value 
functions and $Q_f$ and $Q_g$ be the respective optimal policies. Then, 
 \[
  \max_{x\in\X} \onenorm{Q_f(\cdot|x) - Q_g(\cdot|x)} \le \frac 12 \spannorm{v_f-v_g}.
 \]
\end{lemma}
\begin{proof}
 Let us study the relative entropy between $Q_f(\cdot|x)$ and $Q_g(\cdot|x)$ for any fixed $x$:
 \[
 \begin{split}
  \kl{Q_f(\cdot|x)}{Q_g(\cdot|x)} =& \sum_y Q_f(y|x) \log\frac{Q_f(y|x)}{Q_g(y|x)}
  \\
  =& \sum_y Q_f(y|x) \pa{v_g(y) - v_f(y)} + \log\frac{\sum_y P(y|x)e^{-v_g(y)}}{\sum_y P(y'|x)e^{-v_f(y')}}
  \\
  =& \sum_y Q_f(y|x) \pa{v_g(y) - v_f(y)} + \log\frac{\sum_y P(y|x)e^{-v_f(y)}}{\sum_y P(y'|x)e^{-v_f(y')}}{e^{-v_g(y)+v_f(y)}}
  \\
  \le& \frac{\spannorm{v_g - v_f}^2}{8},
 \end{split}
 \]
 where the second equality follows from straightforward calculations and the last step follows from Hoeffding's lemma (see, e.g., Lemma~A.1 
in \citealt{CBLu06:Book}). Now the statement of the lemma follows from applying Pinsker's inequality.
\end{proof}
Now, we can conclude the proof of Lemma~\ref{lem:change} by combining Lemmas~\ref{lem:c_to_v} and~\ref{lem:v_to_Q} with the easily-seen fact
$\infnorm{\bc_{t+1}-\bc_t}\le 1/t$.
\jmlrQED

\subsection{The proof of Lemma~\ref{lem:price}}\label{app:price}
Observe that, by the definition of the algorithm and the form of $Q_t$ given by~\eqref{eq:optQ}, we have
\[
 \kl{Q_{t}(\cdot|x)}{P(\cdot|x)} = v_t(x) + \lambda_t - \bc_{t}(x) - \sum_{x'} Q_t(x'|x) v_t(x'),
\]
which, by using $\tloss_t = f(\cdot;c_t)$ and the form of $f$ given in Eq.~\eqref{eq:statcost} implies that 
\begin{equation}\label{eq:lpiform}
\begin{split}
 \tloss_t(\pi_t) &= \sum_{x} \mu_t(x) \pa{c_t(x) + v_t(x) + \lambda_{t} - \bc_{t}(x) - \sum_{x'} Q_t(x'|x) v_t(x')}
 \\
 &= \sum_{x} \mu_t(x) \pa{c_t(x) + \lambda_{t} - \bc_{t}(x)},
\end{split}
\end{equation}
where the second equality follows from the fact that $\sum_{x} \mu_t(x) Q_t(x'|x) = \mu_t(x')$. 
Combining this with the analogous expression for $\tloss_t(\pi_{t+1})$, we get
\begin{equation}\label{eq:lpidiff}
\begin{split}
 \tloss_t(\pi_{t}) - \tloss_t(\pi_{t+1}) &= \pa{\mu_{t} - \mu_{t+1}}\transpose c_t + \mu_t\transpose \bc_t - 
 \mu_{t+1}\transpose \bc_{t+1} + \lambda_{t} - \lambda_{t+1}
 \\
 &= \pa{\mu_{t} - \mu_{t+1}}\transpose \pa{c_t + \bc_{t}} + \mu_{t+1}\transpose \pa{\bc_{t} - \bc_{t+1}}
 + \lambda_{t} - \lambda_{t+1}
 \\
 &\le 2 \onenorm{\mu_{t+1} - \mu_t} + \infnorm{\bc_{t+1} - \bc_{t}} + \lambda_{t} - \lambda_{t+1}.
\end{split}
\end{equation}
It remains to bound the last two terms. Defining $\Delta_t = \bc_{t+1} - \bc_{t}$, we have
\[
\begin{split}
 \lambda_{t+1} &= \min_\pi f(\pi;\bc_{t+1}) = \min_\pi \Bpa{f(\pi;\bc_{t}) + \sum_{x,x'} \pi(x,x') \Delta_t(x)}
 \\
 &\ge \min_\pi f(\pi;\bc_{t}) + \min_{\pi'} \sum_{x,x'} \pi'(x,x') \Delta_t(x) \ge \lambda_{t} - \infnorm{\Delta_t},
\end{split}
\]
where the second equality uses the form of $f$, and the last step uses the fact that $\pi'$ is a probability distribution over 
$\X\times\X$. This gives
\[
 \tloss_t(\pi_{t}) - \tloss_t(\pi_{t+1}) \le 2 \onenorm{\mu_{t+1} - \mu_t} + 2 \infnorm{\Delta_t}.
\]

It remains to bound $\onenorm{\mu_{t+1} - \mu_t}$. A simple argument (see, e.g., Lemma~4 of \citealt{neu14o-mdp-full}) shows that 
\[
\onenorm{\mu_{t+1} - \mu_t} \le 
\max\ev{\tau(Q_t),\tau(Q_{t+1})} \max_x \onenorm{Q_{t+1}(\cdot|x)-Q_t(\cdot|x)}.
\]
Now, the first factor can be bounded by $\tau$ and the second by appealing to Lemma~\ref{lem:change}. The proof is 
concluded by plugging the above bounds into Equation~\eqref{eq:lpidiff}, using $\infnorm{\Delta_t} \le 1/t$, summing up both sides, and 
noting that $\sum_{t=1}^T 1/t \le 1 + 
\log T$. \jmlrQED

\subsection{The proof of inequality~\eqref{eq:mu_to_p}}\label{app:mu_to_p}
We will now prove the inequality
\[
\onenorm{\mu_t - p_t} \le 2 e^{-\pa{t-1}/\tau} +  \tau \pa{\tau+1} \sum_{n=1}^{t-1} \frac{e^{-(t-n)/\tau}}{n}.
\]
If $t=1$, the inequality clearly holds as $\onenorm{\mu_1 - p_1} \le 2$. We let $\varepsilon_t = \max_x \onenorm{Q(\cdot|x) - 
Q_{t-1}(\cdot|x)}$. 
By the triangle inequality, we have
\[
\begin{split}
 \onenorm{p_t - \mu_t} &\le \onenorm{p_t - \mu_{t-1}} + \onenorm{\mu_{t-1} - \mu_t}
 \\
 &\le e^{-1/\tau} \onenorm{p_{t-1} - \mu_{t-1}} + \pa{\tau+1} \varepsilon_t,
\end{split}
\]
where we used the fact that $p_t = p_{t-1}\transpose Q_t$, and $\onenorm{\mu_{t-1} - \mu_t} \le \pa{\tau+1} \varepsilon_t$, which follows 
from Lemma~4 of \citet{neu14o-mdp-full}. Continuing recursively, we obtain
\[
 \begin{split}
  \onenorm{p_t - \mu_t} 
  &\le e^{-1/\tau} \pa{e^{-1/\tau} \onenorm{p_{t-2} - \mu_{t-2}} + \pa{\tau+1} \varepsilon_{t-1}} + \pa{\tau+1}
\varepsilon_t
\\
  &\vdots
  \\
  &\le e^{-{t-1}/\tau} \onenorm{p_{1} - \mu_{1}} + \pa{\tau+1} \sum_{n=1}^{t} \varepsilon_n e^{-\pa{t-n}/\tau}
  \\
  &\le e^{-{t-1}/\tau} \onenorm{p_{1} - \mu_{1}} + \pa{\tau+1} \sum_{n=1}^{t} \frac{\tau e^{-\pa{t-n}/\tau}}{n},
 \end{split}
\]
where the last inequality follows from $\varepsilon_t \le \tau/t$, which holds by Lemma~\ref{lem:change}.

\end{document}